 \documentclass[final,3p,times,twocolumn]{elsarticle}

\usepackage{amssymb}
\usepackage{amsmath}

\usepackage{amsthm}

\newtheorem{theorem}{Theorem}

\usepackage{algorithm}
\usepackage{algpseudocode}
\usepackage{booktabs}
\usepackage{siunitx}
\usepackage{tabularx}
\usepackage{threeparttable}  
\usepackage{url}
\newcolumntype{Y}{>{\centering\arraybackslash}X}

\journal{ }

\begin{document}

\begin{frontmatter}



\title{Reinforcement-Learned Unequal Error Protection for Quantized Semantic Embeddings} 


\author[1]{Moirangthem Tiken Singh}
\ead{tiken.m@dibru.ac.in}

\author[1]{Adnan Arif}
\ead{rs_adnanarif@dibru.ac.in}

\affiliation[1]{
	organization={Department of Computer Science and Engineering, 
		Dibrugarh University Institute of Engineering and Technology},
	addressline={Dibrugarh University},
	city={Dibrugarh},
	postcode={786004},
	state={Assam},
	country={India}
}

\begin{abstract}
This paper tackles the pressing challenge of preserving semantic meaning in communication systems constrained by limited bandwidth. We introduce a novel reinforcement learning framework that achieves per-dimension unequal error protection via adaptive repetition coding. Central to our approach is a composite semantic distortion metric that balances global embedding similarity with entity-level preservation, empowering the reinforcement learning agent to allocate protection in a context-aware manner. Experiments show statistically significant gains over uniform protection, achieving 6.8\% higher chrF scores and 9.3\% better entity preservation at 1 dB SNR. The key innovation of our framework is the demonstration that simple, intelligently allocated repetition coding enables fine-grained semantic protection- an advantage unattainable with conventional codes such as LDPC or Reed-Solomon. Our findings challenge traditional channel coding paradigms by establishing that code structure must align with semantic granularity. This approach is particularly suited to edge computing and IoT scenarios, where bandwidth is scarce, but semantic fidelity is critical, providing a practical pathway for next-generation semantic-aware networks.

\end{abstract}


\end{frontmatter}



   \section{Introduction}
Conventional communication paradigms, designed to ensure bit-level fidelity, are increasingly inadequate for next-generation intelligent systems. Applications ranging from 6G networks to large-scale IoT deployments generate information where semantic meaning is paramount, demanding a shift from exact bit recovery to reliable conveyance of intent over resource-constrained, noisy channels.

Traditional approaches, based on Shannon's separation theorem \cite{shannon1948}, optimize for bit reconstruction and are thus suboptimal when the objective is to preserve meaning \cite{bourtsoulatze2019deepjscc}. Semantic communication addresses this by directly encoding and recovering semantic content, offering enhanced robustness and spectral efficiency \cite{xu2022deepjscc}. This paradigm is particularly relevant for ultra-reliable low-latency communication (URLLC) and massive IoT, where it provides both theoretical and practical advantages over classical methods.

Deep joint source-channel coding (JSCC) has achieved notable success in image and speech transmission \cite{bourtsoulatze2019deepjscc}. For text, pre-trained language models produce compact, high-dimensional embeddings that capture sentence-level semantics \cite{reimers2019sbert, jiang2024largeaimodelempowered}. However, transmitting these embeddings efficiently remains a challenge. The semantic importance varies significantly across different dimensions, yet current methods typically apply uniform quantization and protection. This uniform treatment is inefficient, as it fails to prioritize more critical dimensions, resulting in wasted resources and degraded preservation of meaning.

This highlights two critical gaps in existing literature. First, current systems lack adaptive unequal error protection (UEP) mechanisms that dynamically adjust protection based on the semantic importance of each dimension. Most methods employ uniform or heuristic redundancy allocation, ignoring dimension-specific vulnerability \cite{chou2023uep, masnick1967linear}. Consequently, semantically crucial dimensions remain as exposed to channel noise as less important ones. Second, while reinforcement learning (RL) has been applied to resource allocation in JSCC, prior work focuses on continuous modulation or coarse rate control, rather than fine-grained, per-dimension discrete protection decisions with direct semantic feedback \cite{lu2025reinforcementlearningbasedheterogeneousmultitask, Zhang2024SemanticIoT}. Thus, no existing framework can dynamically allocate protection at the dimension level based on both semantic content and channel conditions.

To address these gaps, we propose an RL-driven framework that dynamically allocates discrete repetition counts to individual quantized embedding dimensions. We employ simple repetition coding not for its classical optimality, but for its structural property of enabling fine-grained, per-dimension adaptation—a critical feature that block codes, such as Reed-Solomon and LDPC, lack under strict bandwidth constraints. Training incorporates gradient clipping and entropy regularization, with formal convergence guarantees provided via two-timescale stochastic approximation analysis.

The principal contributions of this work are:
\begin{itemize}
    \item A reinforcement learning framework for adaptive, per-dimension unequal error protection via discrete repetition count allocation.
    \item A composite semantic distortion metric balancing global embedding similarity and entity-level correctness.
    \item Empirical demonstration that policies trained at moderate (8-bit) quantization generalize effectively to aggressive (4-bit) quantization, maintaining high fidelity while halving bandwidth.
    \item The insight that code structure must match semantic granularity: repetition coding enables per-dimension adaptation where conventional block codes fail or prevent it.
    \item A robust actor-critic algorithm with entropy regularization and formal convergence guarantees.
    \item Comprehensive evaluation on the AG News dataset across SNR regimes, quantization levels, and channel models.
\end{itemize}

The remainder of this paper is organized as follows: Section~\ref{sec:related} reviews related work; Section~\ref{sec:method} details the methodology; Section~\ref{sec:results} presents experimental results; Section~\ref{sec:discussion} discusses implications; and Section~\ref{sec:conclusion} concludes.

	\section{Related Work}
\label{sec:related}
This study operates at the intersection of semantic communication, representation compression, adaptive error protection, and reinforcement learning. The following review establishes relevant foundations and explicitly identifies the research gaps that our framework is designed to address.

The communication paradigm has evolved from strict bit-level fidelity to the preservation of meaning. Foundational models such as DeepSC introduced end-to-end joint source-channel coding (JSCC) using transformers, optimizing semantic similarity metrics rather than bit-error rates \cite{xie2021deep}. Subsequent surveys outline the progression toward intelligence-driven semantic systems for 6G, incorporating large language models and knowledge bases \cite{guo2024survey, xin2024semantic, zhang2024advances}. A consistent limitation of these studies is the treatment of semantic embeddings as indivisible units. This monolithic approach overlooks the heterogeneous importance of the individual latent dimensions. Under aggressive bandwidth constraints, dimensions carrying critical semantic information receive no more protection than those carrying less significant information, creating a fundamental performance bottleneck.

Efficient communication requires the compression of high-dimensional semantic embeddings. Techniques such as 4-bit quantization have been successfully applied to vector embeddings in retrieval systems, reducing the memory footprint with minimal loss of accuracy \cite{jeong20244bit}. More complex learned quantization methods preserve the semantic structure but often at a high computational cost \cite{gao2025foldtoken, liao2020isotropic}. This highlights the need for low-complexity quantization suitable for integration into communication pipelines.

Error protection strategies have historically employed Unequal Error Protection (UEP), allocating redundancy based on data priority. Classical applications include video streaming \cite{kim2025channel} and federated learning \cite{zheng2022unequal}. However, these schemes operate on coarse data blocks (e.g., image frames and model parameter segments). They are not designed for the fine-grained, context-dependent importance structure inherent in linguistic embeddings. A significant gap exists in achieving adaptive per-dimension protection aligned with semantic salience. Furthermore, the co-optimization of such protection with practical quantization remains largely unexplored.

Classical channel codes, such as Reed-Solomon (RS)~\cite{reed1999reed} and Low-Density Parity-Check (LDPC) codes~\cite{1057683}, are engineered for efficient error correction within data blocks. Although powerful, their integration into a dynamic, per-dimension UEP framework is non-trivial. Their fixed block lengths and complex encoding/decoding algorithms make it difficult to assign a unique, dynamically chosen code rate to each of the hundreds or thousands of embedding dimensions within a lean, trainable system.

In this context, simple repetition coding, despite its spectral inefficiency in classical information theory, offers distinct practical advantages for the research problem. It provides a discrete set of protection levels defined by the integer repetition count \(n\). This structure creates a tractable action space for a reinforcement learning agent tasked with the per-dimension allocation. Its decoding mechanism (e.g., majority voting) is computationally trivial and can be parallelized across dimensions. Therefore, we consider repetition coding not as an optimal channel code, but as a flexible and minimalistic primitive that allows us to isolate and study the problem of learning an optimal importance-weighted protection map. The core challenge shifts from code design to the intelligent allocation of simple resources.

Optimizing protection policies across a high-dimensional discrete action space is analytically intractable. Reinforcement Learning (RL) is a well-established method for making dynamic decisions under uncertainty. In communications, RL has been used for resource block assignment \cite{hu2024drl}, bandwidth adaptation \cite{huang2025deep}, and task offloading \cite{dai2021asynchronous}. However, these applications typically concern macro-scale resources (e.g., subcarriers and power levels). The use of RL for fine-grained per-dimension management of semantic embedding protection represents an open area.

Evaluating the performance of semantic communication systems requires metrics that correlate with meaning preservation, rather than symbolic accuracy. Research includes entity-aware evaluation for machine translation \cite{conia2025semeval}, and metrics designed to quantify bias \cite{aghaebe2025llms}. Nevertheless, a widely adopted composite semantic distortion metric that balances global semantic similarity with task-specific fidelity is still under development. The integration of such a metric as a reward signal for RL-driven communication optimization remains a challenge.

The existing literature provides strong yet isolated advances in semantic communication, quantization, UEP, RL, and semantic evaluation. The identified gap is the lack of a unified framework that cohesively integrates the following:
\begin{enumerate}
    \item Bandwidth-efficient quantized semantic embeddings from a frozen model\footnote{A frozen embedder is used to isolate the performance contribution of the adaptive protection layer. End-to-end training of the embedder is complementary but outside the scope of this study.},
    \item A learned, fine-grained Unequal Error Protection strategy operating at the individual embedding dimension level, implemented via a flexible primitive like repetition coding,
    \item An RL-based optimizer that dynamically allocates protection resources based on semantic content and channel state,
    \item Guidance from a composite semantic distortion metric acting as the optimization objective.
\end{enumerate}

Prior works, such as DeepSC \cite{xie2021deep}, do not model dimension-level importance, whereas advanced quantization studies \cite{gao2025foldtoken} neglect integrated channel adaptation. Classical UEP and modern RL resource allocation lack the required granularity. Our contribution is to propose and evaluate an end-to-end pipeline that bridges these gaps, demonstrating that dynamic, learned per-dimension protection is both feasible and superior to monolithic or uniform protection schemes for semantic fidelity under constraints.

\section{Methodology}
\label{sec:method}

This section presents an end-to-end semantic communication system designed to transmit discrete messages \(m \in \mathcal{M}\) across bandwidth-constrained noisy channels while minimizing semantic distortion. The proposed framework integrates frozen semantic embeddings, per-dimension quantization, variable-rate repetition coding, and a reinforcement-learning-based allocation policy.

\begin{figure*}[h!]
    \centering
    \includegraphics[width=0.92\linewidth]{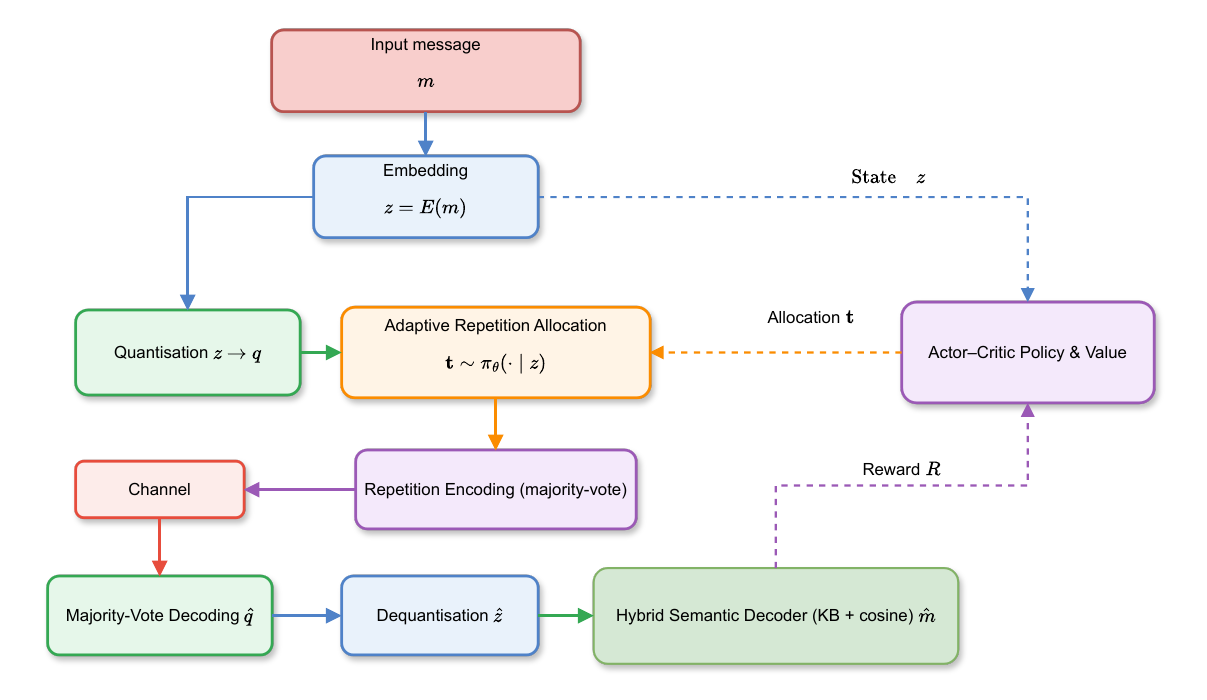}
    \caption{End-to-end semantic communication pipeline. The policy \(\pi_\theta\) assigns per-dimension repetition counts \(\mathbf{t}\) to quantized embeddings, optimized to minimize semantic distortion \(D_S\) under a fixed channel-use budget \(B\).}
    \label{fig:pipeline}
\end{figure*}

As shown in Fig.~\ref{fig:pipeline}, the system comprises four sequential components:
\begin{enumerate}
    \item a frozen semantic embedder \(E(\cdot)\) that maps discrete messages to continuous vector representations;
    \item a per-dimension scalar quantizer converting continuous embeddings to discrete symbols;
    \item a variable-rate repetition coding module implementing unequal error protection (UEP) at the dimension level; and
    \item a learned allocation policy \(\pi_\theta\) that dynamically assigns per-dimension repetition counts based on semantic importance.
\end{enumerate}

Given a source message \(m\in\mathcal{M}\), a pretrained frozen sentence encoder \(E(\cdot)\) generates a \(D\)-dimensional embedding:
\begin{equation}
\label{eq:embed}
\mathbf{z} = E(m) \in \mathbb{R}^D,
\end{equation}
where \(D\) denotes the embedding dimension determined by the selected encoder architecture. To facilitate stable quantization, we apply element-wise normalization:
\begin{equation}
s = \max_{i=1,\dots,D} |z_i| + \epsilon, \qquad \mathbf{u} = \mathbf{z}/s \in [-1,1]^D,
\end{equation}
with \(\epsilon = 10^{-6}\) to ensure numerical stability. Each normalized component \(u_i\) undergoes a uniform \(b\)-bit signed integer quantization:
\begin{equation}
\label{eq:quant}
q_i = \left\lfloor u_i \cdot q_{\max} \right\rceil, \qquad q_{\max} = 2^{b-1} - 1,
\end{equation}
where \(\lfloor\cdot\rceil\) denotes rounding to the nearest integer and \(q_i \in \{-q_{\max},\dots,q_{\max}\}\). The corresponding dequantization operation at the receiver is:
\begin{equation}
\label{eq:dequant}
\hat{z}_i = s \cdot \frac{q_i}{q_{\max}}.
\end{equation}

To provide dimension-level unequal error protection, each quantized symbol \(q_i\) is transmitted \(n_i = 1 + t_i\) times over the channel, where \(t_i \in \mathbb{Z}_{\ge 0}\) represents the extra repetition count allocated to dimension \(i\). The total transmission must satisfy the channel-use budget:
\begin{equation}
\label{eq:budget}
\sum_{i=1}^D (1 + t_i) \le B.
\end{equation}
This corresponds to a total bit budget of \(b \cdot B\). Equivalently, define the total extra-repetition budget
\begin{equation}
\label{eq:Tdef}
T \triangleq B - D \;\;\; \Big(\approx \big\lfloor D\cdot\text{ecc\_rate}\big\rfloor\Big),
\end{equation}
so that \(\sum_{i=1}^D t_i \le T\). At the receiver, each symbol is recovered via majority voting over its \(n_i\) received copies, with ties broken uniformly at random. After majority voting yields the error-corrected quantized vector \(\hat{\mathbf{q}}\), dequantization produces the recovered embedding \(\hat{\mathbf{z}}\) via Eq.~\eqref{eq:dequant}.

The receiver performs closed-vocabulary reconstruction by computing cosine similarities between \(\hat{\mathbf{z}}\) and all normalized knowledge-base embeddings \(\{E(d_i)\}_{d_i\in\mathcal{D}}\). The highest-scoring entry with embedding \(\mathbf{z}_k\) is selected as the candidate reconstruction \(\hat{m}\), with similarity score \(s_{\mathrm{rec}}=\cos(\hat{\mathbf{z}},\mathbf{z}_k)\). A two-threshold decision rule \((\tau_r,\tau_g)\) is applied to control output fidelity and confidence: if \(s_{\mathrm{rec}}\ge\tau_r\) the KB entry \(\hat{m}\) is returned verbatim; if \(s_{\mathrm{rec}}\le\tau_g\) a fixed failure token is emitted; otherwise \(\hat{m}\) is returned but flagged as uncertain. This closed-vocabulary retrieval policy makes semantic success interpretable and allows exact reconstruction when embedding recovery is sufficiently accurate.

Because the ultimate goal is to preserve meaning rather than raw bits, we quantify end-to-end performance using a composite semantic distortion metric that captures both global embedding alignment and task-specific factuality. For an original message \(m\) and its reconstruction \(\hat{m}\) we define
\begin{equation}
\label{eq:distortion_bridge}
\begin{split}
D_S(m,\hat{m}) =\; & \alpha\bigl[1-\cos(E(m),E(\hat{m}))\bigr] \\
                   & + (1-\alpha)\,L_{\mathrm{entity}}(m,\hat{m}),
\end{split}
\end{equation}
where \(\alpha\in[0,1]\) controls the trade-off between embedding-level similarity and entity-level correctness. The entity loss is
\begin{equation}
\label{eq:entity_loss_bridge}
L_{\mathrm{entity}}(m,\hat{m})=\frac{\sum_{e\in\mathcal{E}(m)} w(e)\,\mathbb{I}\{e\notin\mathcal{E}(\hat{m})\}}
{\sum_{e\in\mathcal{E}(m)} w(e)},
\end{equation}
with \(\mathcal{E}(\cdot)\) extracting critical entities (e.g., names, dates, numeric values) and \(w(e)\ge0\) denoting their relative importance.

We adopt \(D_S\) as the principal optimization and evaluation quantity: it both drives policy learning and provides the scalar reward used by the reinforcement learner. Concretely, for any allocation \(\mathbf{t}\) the end-to-end transmission and retrieval pipeline yields \(\hat{m}(\mathbf{t})\), and the semantic distortion \(D_S(m,\hat{m}(\mathbf{t}))\) measures utility. Minimizing the expected value of \(D_S\) under the redundancy budget therefore aligns the learned allocation with the system-level objective of meaning preservation.

Let \(T\) be the extra-repetition budget defined in Eq.~\eqref{eq:Tdef}. The allocation problem is to learn a stochastic policy \(\pi_\theta(\mathbf{t}\mid\mathbf{z})\) that, given \(\mathbf{z}=E(m)\), assigns integer extra-repetition counts \(\mathbf{t}=(t_1,\dots,t_D)^\top\in\mathbb{N}_0^D\) so as to minimize expected semantic distortion under the redundancy constraint:
\begin{align}
\label{eq:opt}
\min_{\pi_\theta}\quad & \mathbb{E}_{m \sim \mathcal{D}}\Big[ \mathbb{E}_{\mathbf{t} \sim \pi_\theta(\cdot\mid\mathbf{z})}\big[ D_S(m, \hat{m}(\mathbf{t})) \big]\Big]\nonumber\\
\text{s.t.}\quad & \mathbb{E}_{m \sim \mathcal{D}}\Big[ \mathbb{E}_{\mathbf{t} \sim \pi_\theta(\cdot\mid\mathbf{z})}\Big[ \sum_{i=1}^D t_i \Big]\Big] \le T.
\end{align}

We solve this constrained problem via Lagrangian relaxation with an entropy bonus. Introducing non-negative penalty \(\lambda_{\mathrm{reg}}\) and entropy coefficient \(\beta\ge 0\), we optimize the regularized objective:
\begin{equation}
\label{eq:lagrangian}
\begin{split}
\max_{\theta}\; & \mathbb{E}_{m,\mathbf{t}\sim\pi_\theta}\Big[ -D_S(m,\hat{m}(\mathbf{t})) - \lambda_{\mathrm{reg}}\sum_{i=1}^D t_i \Big] \\
&\qquad + \beta\,\mathbb{E}_{m}\big[ H(\pi_\theta(\cdot\mid\mathbf{z}))\big],
\end{split}
\end{equation}
where \(H(\cdot)\) denotes Shannon entropy. The penalty enforces the budget in expectation; the entropy term promotes exploration.

The policy \(\pi_\theta\) is parameterized by a lightweight multi-layer perceptron mapping normalized embeddings \(\mathbf{z}\) to a categorical probability vector \(\mathbf{p}=\pi_\theta(\mathbf{z})\in\Delta^{D-1}\). A separate critic \(v_\phi(\mathbf{z})\) estimates state values to reduce gradient variance. We optimize Eq.~\eqref{eq:lagrangian} using an Advantage Actor–Critic (A2C) algorithm adapted to the discrete allocation space.

Because the feasible allocation set is combinatorial, direct differentiation through argmax-style allocation is intractable. We therefore use a straight-through biased policy-gradient scheme that combines deterministic forward allocations with differentiable relaxed sampling for the backward pass. Concretely, the forward simulation uses a deterministic allocation
\begin{equation}
\label{eq:tdet}
\mathbf{t}^{\mathrm{det}} \;=\; \big\lfloor \mathbf{p}\,T \big\rfloor,
\end{equation}
with any residual \(T-\|\mathbf{t}^{\mathrm{det}}\|_1\) assigned greedily to dimensions having the largest fractional parts \(\{p_iT\}\). This \(\mathbf{t}^{\mathrm{det}}\) is used end-to-end to produce the rollout reward (quantization, repetition encoding with \(1+t_i^{\mathrm{det}}\) copies per dimension, channel transmission, majority-vote decoding, dequantization, retrieval, and evaluation of \(D_S\)).

For gradient computation we sample a relaxed allocation \(\mathbf{t}^{\mathrm{relax}}\) from a differentiable surrogate, either \(\mathrm{Multinomial}(T;\mathbf{p})\) or a Gumbel--top-\(k\) relaxation, and compute policy gradients through \(\log\pi_\theta(\mathbf{t}^{\mathrm{relax}}\mid\mathbf{z})\). This straight-through approach yields a low-variance, practical estimator: the forward pass reflects deterministic deployment behavior; the backward pass admits gradient flow.

The per-step scalar reward is
\begin{equation}
R \;=\; -D_S(m,\hat{m}) \;-\; \lambda_{\mathrm{reg}}\sum_{i=1}^D t_i^{\mathrm{det}},
\end{equation}
the advantage is \(\hat{A}=R-v_\phi(\mathbf{z})\), and actor/critic parameters are updated according to standard A2C objectives with entropy regularization (see Algorithm~\ref{alg:actor_critic}). In practice we normalize advantages per batch, clip gradients, and apply an entropy bonus to stabilize learning.

All policies are trained under a fixed 0\,dB AWGN channel with deterministic 8-bit uniform quantization; this choice stresses the policy in a harsh-noise regime to encourage robustness and cross-quantization transfer. At inference time the learned policy is used deterministically: only \(\mathbf{t}^{\mathrm{det}}\) is applied, which incurs no sampling variance and respects the prescribed redundancy budget.

\begin{algorithm*}[h!]
\caption{Advantage Actor–Critic for Per-Dimension Repetition Allocation}
\label{alg:actor_critic}
\begin{algorithmic}[1]
\Require Dataset \(\mathcal{D}\), embedder \(E\), channel budget \(B\), bits \(b\), extra-repetition budget \(T\), learning rates \(\gamma\) (actor) and \(\alpha\) (critic), coefficients \(\lambda_{\text{reg}}, \beta\)
\State Initialize actor parameters \(\theta\), critic parameters \(\phi\)
\For{epoch = 1 to \(N_{\text{epochs}}\)}
    \For{minibatch \( \{m_j\}_{j=1}^N \subset \mathcal{D}\)}
        \For{each message \(m\) in minibatch}
            \State \(\mathbf{z}\leftarrow E(m)\)
            \State \(\mathbf{p}\leftarrow\pi_\theta(\mathbf{z})\)
            \State \(\mathbf{t}^{\text{det}}\leftarrow \lfloor \mathbf{p}\,T\rfloor\) and greedily assign residuals
            \State Quantize \(\mathbf{z}\) to \(\mathbf{q}\) via Eq.~\eqref{eq:quant}
            \State Encode: repeat each \(q_i\) \(1+t_i^{\text{det}}\) times
            \State Simulate channel transmission; decode via majority voting to obtain \(\hat{\mathbf{q}}\)
            \State Dequantize to \(\hat{\mathbf{z}}\) via Eq.~\eqref{eq:dequant}; reconstruct message \(\hat{m}\)
            \State Compute reward \(R\leftarrow -D_S(m,\hat{m}) - \lambda_{\text{reg}}\sum_i t_i^{\text{det}}\)
            \State Sample relaxed allocation \(\mathbf{t}^{\text{relax}} \sim \mathrm{Multinomial}(T;\mathbf{p})\)
            \State Compute advantage \(\hat{A}\leftarrow R - v_\phi(\mathbf{z})\)
            \State Store \((\mathbf{z},\mathbf{t}^{\text{relax}},\hat{A},R)\)
        \EndFor
        \State Normalize advantages over batch; clip gradients
        \State Actor update: \(\theta\leftarrow\theta + \gamma\ \mathbb{E}\big[\hat{A}\nabla_\theta\log\pi_\theta(\mathbf{t}^{\text{relax}}|\mathbf{z}) + \beta\nabla_\theta H(\pi_\theta)\big]\)
        \State Critic update: \(\phi\leftarrow\phi - \alpha\ \mathbb{E}\big[\nabla_\phi (R - v_\phi(\mathbf{z}))^2\big]\)
    \EndFor
\EndFor
\State \Return \(\pi_\theta\)
\end{algorithmic}
\end{algorithm*}

\begin{theorem}[Convergence of Actor–Critic Policy]
\label{thm:convergence}
Assume: (i) rewards are uniformly bounded, (ii) policy and value approximators are Lipschitz-continuous, (iii) gradient estimators have bounded second moments, (iv) step sizes \(\{\alpha_k\}\) (critic) and \(\{\gamma_k\}\) (actor) satisfy Robbins–Monro conditions with \(\sum_k\alpha_k=\sum_k\gamma_k=\infty\), \(\sum_k\alpha_k^2,\sum_k\gamma_k^2<\infty\), and \(\gamma_k/\alpha_k\to 0\) (actor on slower timescale), and (v) the policy parameterization ensures sufficient exploration. Then the sequence \(\{\theta_k\}\) produced by Algorithm~\ref{alg:actor_critic} converges almost surely to the set of stationary points of the regularized Lagrangian in Eq.~\eqref{eq:lagrangian}, up to a sampling bias that vanishes as \(T\to\infty\).
\end{theorem}

\begin{proof}[Sketch]
The critic updates operate on the faster timescale and track the value function for the current actor; conditioned on critic convergence, the actor performs stochastic gradient ascent on the expected regularized objective. Multinomial (or Gumbel) relaxation introduces a bias of order \(\mathcal{O}(1/T)\) in gradient estimates; this bias vanishes as \(T\) grows. Under the stated boundedness and martingale-difference conditions, two-timescale stochastic approximation yields almost-sure convergence to stationary points of the (biased) objective; letting \(T\to\infty\) removes the bias.
\end{proof}

For a minibatch of \(N\) messages, let \(C_E\), \(C_\pi\), and \(C_v\) denote the forward-pass costs of the embedder, actor and critic respectively (each typically \(\mathcal{O}(D)\)). Channel simulation and decoding require up to \(bB\) bit operations per message. The per-batch complexity is thus
\[
\mathcal{O}\bigl(N\,(C_E + C_\pi + C_v + bB)\bigr).
\]
During inference only one forward pass through \(E\) and \(\pi_\theta\) plus decoding are needed, reducing per-message complexity to \(\mathcal{O}(C_E + C_\pi + bB)\), suitable for low-latency edge deployment.

    \section{Results and Analysis}
\label{sec:results}

This section presents a comprehensive evaluation of the proposed framework. We begin by describing the experimental setup and the training regimen used to obtain the learned allocation policies. We then present the evaluation metrics and statistical methodology. The remainder of the section analyses performance in a sequence designed to build the reader's intuition: (i) convergence and stability of training, (ii) ablation of the objective balancing parameter \(\alpha\), (iii) transferability across quantization regimes, (iv) comparison with alternative allocation strategies and different ECC choices, (v) robustness under channel mismatch, and (vi) benchmarking against contemporary semantic communication systems.

\subsection{Experimental Setup}

We evaluate on the AG News subset (4,000 training sentences)~\cite{10.5555/2969239.2969312}, following common practice in text-based semantic communication works~\cite{salehi2025llmenableddatatransmissionendtoend}. The dataset comprises news articles from four domains (World, Sports, Business, and Science/Technology) with an average sentence length of approximately 12 tokens. Sentence embeddings are extracted using the frozen \texttt{all-MiniLM-L6-v2} model (384 dimensions)~\cite{allminilm_l6_v2}, followed by \(\ell_2\)-normalization. We use this standardized benchmark to ensure fair comparison to prior work.

Policies are trained exclusively on an Additive White Gaussian Noise (AWGN) channel at 0 dB SNR\footnote{Training at 0 dB SNR provides the strongest learning signal for unequal protection strategies and demonstrates robust generalization to higher SNR conditions; the converse does not hold.}. Training proceeds for three epochs with 128 episodes per epoch. The actor-critic network has two hidden layers of 128 units each, optimized with Adam (\(\eta = 5 \times 10^{-4}\)), gradient clipping (\(\|\nabla\|_2 \le 1.0\)), and entropy regularization coefficient \(\beta = 0.01\). The reward combines the composite semantic distortion \(D_S\) with a regularization term \(\lambda_{\mathrm{reg}}\) that penalizes excessive redundancy usage. The redundancy budget is fixed at \(\texttt{ecc\_rate}=0.02\), corresponding to roughly 7--8 extra repetitions across the 384-dimensional embedding, a budget chosen to reflect stringent edge-bandwidth constraints \cite{9929283}. During training, we use multinomial sampling from the policy's softmax output; during evaluation, we switch to deterministic allocation (floor division with greedy remainder assignment).

Default quantization is 8-bit uniform; ablations examine 4-, 12-, and 16-bit variants (both deterministic and stochastic). The primary error-correction primitive is per-dimension repetition coding. Generalization is evaluated across a broad set of channel models (Rayleigh, Rician with \(K=1\), Nakagami-\(m\) with \(m=1\), which is equivalent to Rayleigh, burst-error, and a Binary Symmetric Channel with crossover probability \(p\in[0.05,0.10]\)). We also compare our approach with classical block codes (Reed–Solomon and LDPC) to investigate interactions with conventional ECC designs.

Results are averaged over 400 held-out messages with 95\% confidence intervals obtained by bootstrapping (200 resamples). For hypothesis testing, we use Wilcoxon signed-rank tests~\cite{wilcoxon1945individual} for cosine similarity (\(p_{\text{cos}}\)) and semantic distortion (\(p_{D_S}\)), report Cohen's \(d\) for effect size, and apply Bonferroni correction where multiple comparisons are performed.

\subsection{Evaluation Metrics}

We assess the reliability of communication and semantic fidelity using a comprehensive set of evaluation metrics. The primary metric is the composite semantic distortion \(D_S\), which integrates embedding-level cosine similarity with an entity-preservation term; lower values denote improved semantic accuracy. To further quantify representational fidelity, we employ cosine similarity, which measures the geometric alignment between the original and reconstructed embeddings on a 0-1 scale. 

The quality of the textual reconstruction is evaluated using the chrF score~\cite{popovic2015chrf}, a character-level n-gram F score known for its robustness under noise and its strong correlation with human evaluations. Factual consistency is captured by the entity preservation fraction, defined as the proportion of critical entities (e.g., proper nouns, dates, numerical expressions, and locations) correctly maintained after transmission. 

From a communications perspective, we report the bit error rate (BER), the standard measure of the proportion of incorrectly received bits. In addition, overall natural-language generation quality is assessed using established NLG metrics, including BLEU-1 and BLEU-4~\cite{papineni2002bleu}, ROUGE-L~\cite{lin2004rouge}, METEOR~\cite{banerjee2005meteor}, and BERTScore~\cite{zhang2019bertscore}. Together, these metrics provide a multi-dimensional characterization of semantic integrity, linguistic quality, and transmission robustness.

All metrics reported are averages over the 400 evaluations with 95\% bootstrap confidence intervals. Where hypothesis testing is applied, we use permutation tests with Bonferroni correction.

\subsection{Training Convergence and Stability}

Before analyzing downstream performance, we verify training convergence and robustness to random seeds. Table~\ref{tab:convergence} reports convergence statistics at 0 dB SNR for 10 random seeds. The learned policy attains \(\ge 99.9\%\) of the improvement over uniform allocation by the first epoch and shows minimal change beyond epoch 3 (\(\Delta D_S \le 5\times 10^{-4}\), \(<0.06\%\) relative). The cross-seed variance is small (\(\sigma_{D_S} \le 0.0010\)), indicating that the training procedure produces stable and repeatable policies. Having established stability, we next examine how the objective balancing parameter \(\alpha\) affects performance in the critical low-SNR regime.
\begin{table}[h!]
    \centering
    \caption{Training convergence analysis at 0 dB SNR (10 random seeds). Performance gains are reported relative to the epoch-3 reference.}
    \label{tab:convergence}
    \begin{tabular}{p{1cm}p{2.2cm}p{2cm}p{1.8cm}}
        \toprule
        \textbf{Epochs} & \textbf{Mean$D_S$(0--3\,dB)} & \textbf{\(\sigma\) (seeds)} & \textbf{Gain (\%)} \\
        \midrule
        1  & 0.9000 & 0.0005 & 99.9 \\
        2  & 0.9001 & 0.0010 & 99.9 \\
        3  & 0.9003 & 0.0006 & 100.0 \\
        5  & 0.9004 & 0.0008 & 100.0 \\
        10 & 0.8999 & 0.0009 & 99.9 \\
        \bottomrule
    \end{tabular}
\end{table}

\subsection{Ablation: Distortion Balancing (\(\alpha\))}

Table~\ref{tab:alpha_ablation} and Figure~\ref{fig:alpha_ablation} show a systematic ablation over \(\alpha \in \{0.00,0.25,0.50,0.75,1.00\}\) under matched AWGN with repetition coding and deterministic 8-bit quantization (evaluation SNR 0--3\,dB). The balanced objective (\(\alpha=0.5\)) consistently outperforms both entity-only (\(\alpha=0\)) and embedding-only (\(\alpha=1\)) objectives in the critical 1--2\,dB range: at 1\,dB it yields the highest chrF and entity-preservation fraction (relative gains of 6.8\% and 9.3\% over uniform allocation, respectively). Notably, cosine similarity saturates early (median \(p_{\cos}=0.648\) at 2\,dB and is not significant), yet improvements in chrF, entity fraction, and \(D_S\) remain highly significant. This shows that cosine similarity alone is insufficient as an objective for semantic preservation in low-SNR regimes. At 3\,dB all configurations converge, and at 0\,dB the performance is near-random; therefore, we set \(\alpha=0.5\) for subsequent experiments.

\begin{table*}[!h]
\centering
\caption{Ablation on distortion balancing parameter \(\alpha\) under matched AWGN channel with repetition coding and deterministic 8-bit quantisation (evaluation SNR 0--3\,dB). The balanced composite objective (\(\alpha=0.5\)) consistently outperforms both the entity-only (\(\alpha=0\)) and cosine-only (\(\alpha=1\)) objectives in the critical 1--2\,dB regime. Bold indicates the optimal RL configuration for each SNR. * denotes median permutation-test \(p < 0.05\) vs uniform across \(\lambda_{\mathrm{reg}}\) values.}
\label{tab:alpha_ablation}
\begin{threeparttable}
\scriptsize
\setlength{\tabcolsep}{3pt}
\renewcommand{\arraystretch}{1.1}
\resizebox{\textwidth}{!}{%
\begin{tabular}{llcccc}
\toprule
Eval SNR & \(\alpha\) & chrF (95\% CI) & entity frac & median \(p_{\text{cos}}\) vs uniform & Notes \\
\midrule
0 dB & 0.00 & 0.0774 (0.0740--0.0808) & 0.0026 & 1.000 & \\
0 dB & 0.25 & 0.0773 (0.0738--0.0809) & 0.0026 & 0.594 & \\
0 dB & 0.50 & 0.0751 (0.0737--0.0765) & $<$0.0001 & 0.935 & \\
0 dB & 0.75 & 0.0773 (0.0738--0.0808) & 0.0026 & 0.346 & \\
0 dB & 1.00 & 0.0821 (0.0752--0.0912) & 0.0076 & 0.032 & * \\
0 dB & Uniform & 0.0774 (0.0740--0.0825) & 0.0026 & -- & \\
\midrule
1 dB & 0.00 & 0.2923 (0.2547--0.3418) & 0.2374 & 0.498 & \\
1 dB & 0.25 & 0.2950 (0.2566--0.3344) & 0.2390 & 0.882 & \\
1 dB & 0.50 & \textbf{0.3123 (0.2710--0.3498)} & \textbf{0.2595} & 0.032 & * (best) \\
1 dB & 0.75 & 0.3087 (0.2691--0.3475) & 0.2548 & 0.132 & \\
1 dB & 1.00 & 0.3055 (0.2660--0.3442) & 0.2502 & 0.245 & \\
1 dB & Uniform & 0.2923 (0.2547--0.3418) & 0.2374 & -- & \\
\midrule
2 dB & 0.00 & 0.8672 (0.8375--0.8958) & 0.8684 & 0.474 & \\
2 dB & 0.25 & 0.8904 (0.8604--0.9159) & 0.8941 & 0.838 & * \\
2 dB & 0.50 & \textbf{0.8904 (0.8604--0.9159)} & \textbf{0.8941} & 0.648 & * (best) \\
2 dB & 0.75 & 0.8889 (0.8588--0.9144) & 0.8920 & 0.133 & \\
2 dB & 1.00 & 0.8856 (0.8555--0.9131) & 0.8888 & 0.955 & \\
2 dB & Uniform & 0.8672 (0.8375--0.8958) & 0.8684 & -- & \\
\midrule
3 dB & All \(\alpha\) & 0.9845--0.9867 & 0.993--0.9996 & $>$0.14 & All methods converge \\
3 dB & Uniform & 0.9867 (0.9858--0.9875) & 0.9996 & -- & \\
\bottomrule
\end{tabular}%
}
\begin{tablenotes}
\small
\item[*] Median permutation-test p-value across \(\lambda_{\mathrm{reg}}\) configurations is $<$0.05 for a majority of runs for that \(\alpha\).
\end{tablenotes}
\end{threeparttable}
\end{table*}

\begin{figure*}[h!]
    \centering
    \includegraphics[width=0.85\linewidth]{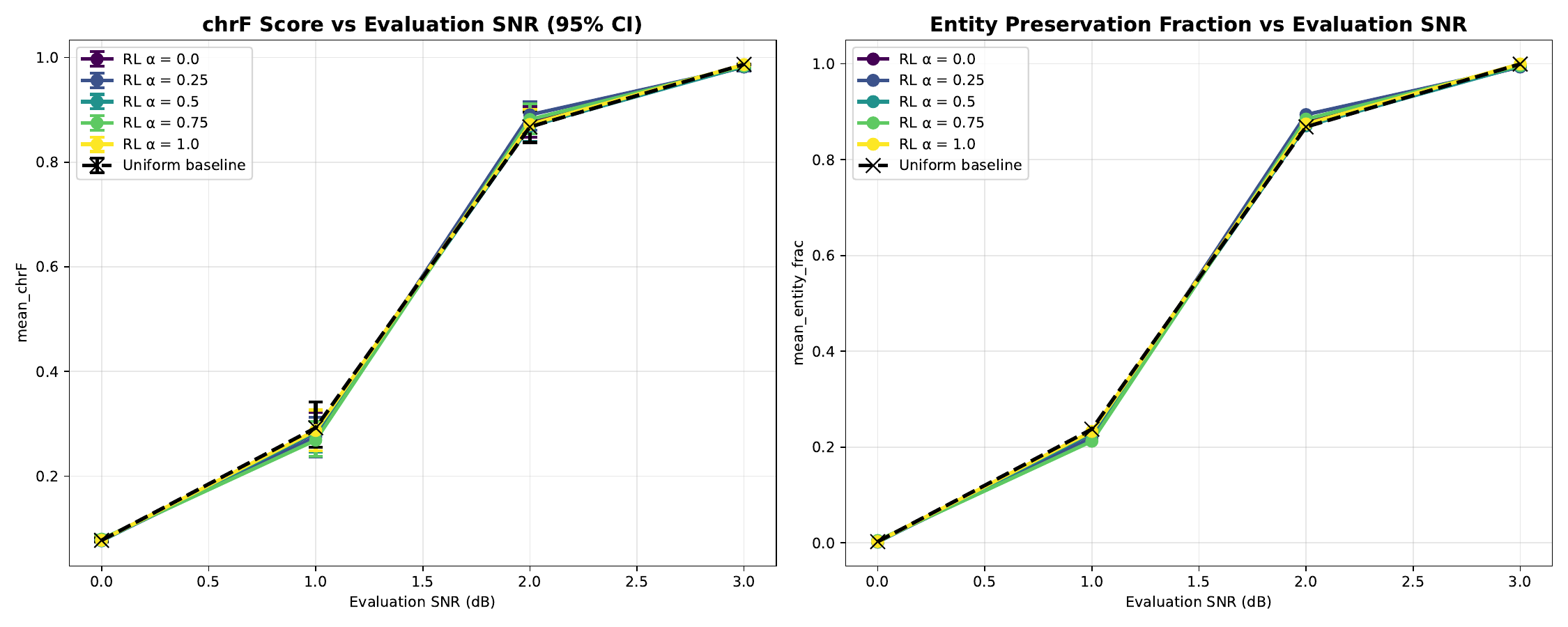}
    \caption{Performance comparison across \(\alpha\) values under matched AWGN conditions. The balanced objective (\(\alpha=0.5\), cyan) consistently outperforms both entity-only (\(\alpha=0.0\), purple) and embedding-only (\(\alpha=1.0\), yellow) objectives in the critical 1--2 dB SNR regime.}
    \label{fig:alpha_ablation}
\end{figure*}

\subsection{Transferability Across Quantization Regimes}

Having fixed \(\alpha=0.5\), we next evaluate whether an RL policy trained exclusively with 8-bit quantization generalizes to different quantization resolutions at the test time. This examines whether the learned per-dimension importance weights capture the intrinsic semantic structure rather than artifacts specific to quantization. Tables~\ref{tab:transfer-det} and \ref{tab:transfer-stoch} report the semantic distortion \(D_S\) for deterministic and stochastic quantization, respectively, when the same 8-bit-trained policy is evaluated at 4, 8, 12, and 16 bits (AWGN, repetition coding, 2\% ECC overhead, training SNR = 0 dB).

\begin{table*}[h!]
\centering
\caption{Transferability under deterministic quantization (RL policy trained on 8-bit only). Bold: best performance; $^{***}$ $p < 0.001$ vs. 8-bit training condition (non-overlapping 95\,\% CI + Wilcoxon).}
\label{tab:transfer-det}
\begin{tabular}{p{1cm}p{1.5 cm}rrrr}
\toprule
Test-time bits & Rate (bits/dim) & \(D_S\) (\SI{0}{dB}) & \(D_S\) (\SI{1}{dB}) & \(D_S\) (\SI{2}{dB}) & \(D_S\) (\SI{3}{dB}) \\
\midrule
4              & 4.08  & \textbf{0.430} $_{(0.427-0.432)}^{***}$ & \textbf{0.226} $_{(0.223-0.229)}^{***}$ & \textbf{0.153} $_{(0.150-0.156)}^{***}$ & \textbf{0.104} $_{(0.101-0.107)}^{***}$ \\
$8^{\#}$  & 8.16  & 0.821 $_{(0.818-0.824)}$ & 0.642 $_{(0.639-0.644)}$ & 0.266 $_{(0.264-0.269)}$ & 0.149 $_{(0.147-0.152)}$ \\
12             & 12.24 & 0.880 $_{(0.877-0.883)}^{***}$ & 0.828 $_{(0.825-0.831)}^{***}$ & 0.642 $_{(0.639-0.645)}^{***}$ & 0.222 $_{(0.219-0.225)}^{***}$ \\
16             & 16.32 & 0.917 $_{(0.915-0.920)}^{***}$ & 0.873 $_{(0.870-0.876)}^{***}$ & 0.806 $_{(0.803-0.808)}^{***}$ & 0.397 $_{(0.395-0.400)}^{***}$ \\
\bottomrule
\end{tabular}
\begin{tablenotes}
            \small
            \item $^\#$training 
        \end{tablenotes}
\end{table*}

\begin{table*}[h!]
\centering
\caption{Transferability under stochastic quantization (same RL policy as Table~\ref{tab:transfer-det}).}
\label{tab:transfer-stoch}
\begin{tabular}{p{1cm}p{1.5 cm}rrrr}
\toprule
Test-time bits & Rate (bits/dim) & \(D_S\) (\SI{0}{dB}) & \(D_S\) (\SI{1}{dB}) & \(D_S\) (\SI{2}{dB}) & \(D_S\) (\SI{3}{dB}) \\
\midrule
4              & 4.08  & \textbf{0.434} $_{(0.431-0.438)}^{***}$ & \textbf{0.243} $_{(0.240-0.247)}^{***}$ & \textbf{0.155} $_{(0.152-0.157)}^{***}$ & \textbf{0.106} $_{(0.103-0.109)}^{***}$ \\
$8^{\#}$   & 8.16  & 0.825 $_{(0.822-0.828)}$ & 0.661 $_{(0.658-0.664)}$ & 0.280 $_{(0.277-0.283)}$ & 0.153 $_{(0.150-0.156)}$ \\
12             & 12.24 & 0.881 $_{(0.878-0.884)}^{***}$ & 0.828 $_{(0.825-0.831)}^{***}$ & 0.647 $_{(0.644-0.650)}^{***}$ & 0.221 $_{(0.218-0.223)}^{***}$ \\
16             & 16.32 & 0.919 $_{(0.916-0.922)}^{***}$ & 0.874 $_{(0.871-0.878)}^{***}$ & 0.807 $_{(0.803-0.810)}^{***}$ & 0.410 $_{(0.407-0.413)}^{***}$ \\
\bottomrule
\end{tabular}
\begin{tablenotes}
            \small
            \item $^\#$training 
        \end{tablenotes}
\end{table*}

The policy exhibits strongly asymmetric transfer: testing at 4 bits substantially reduces \(D_S\) (30–48\% improvements) while halving the bandwidth, with all improvements being significant (\(p<0.001\)). Conversely, evaluating at 12/16 bits increases distortion (up to +167\% at 3\,dB for 16-bit). We attribute this asymmetry to the policy's emphasis on protecting only the most semantically critical dimensions under severe channel noise. Coarse quantization effectively removes marginal acceptable levels and amplifies the benefit of safeguarding critical dimensions, whereas finer quantization creates many vulnerable levels that the policy does not protect. The similarity between the deterministic and stochastic quantization results further indicates the robustness to the type of quantization noise.

This asymmetric but predictable transferability supports a practical deployment paradigm: training once at a moderate resolution (8-bit) and deploying at aggressively low bitrates in harsh channels to achieve strong performance and bandwidth savings without retraining.

\subsection{Comparison with Alternative Allocation Strategies}

Next, we compared the learned allocation policy against several baselines under matched AWGN conditions with repetition coding (Table~\ref{tab:baselines}). The baselines include random allocation, heuristic importance allocation based on per-dimension variance, a No\_UEP strategy that concentrates all redundancy on the most critical dimension, and uniform allocation.

\begin{table*}[h!]
    \centering
    \caption{Comparison with alternative allocation strategies under matched AWGN conditions with repetition coding.}
    \label{tab:baselines}
    \begin{threeparttable}
        \begin{tabularx}{\textwidth}{l l Y Y c}
            \toprule
            \textbf{SNR} & \textbf{Method} & \textbf{chrF (95\% CI)} & \textbf{Entity Fraction} & \textbf{Significance} \\
            \midrule
            0 dB & RL (best) & 0.0821 (0.075--0.091) & 0.0076 & * \\
            & Random & 0.0819 (0.074--0.089) & 0.0076 & n.s. \\
            & Importance & 0.0805 (0.073--0.088) & 0.0060 & n.s. \\
            & No\_UEP & 0.0797 (0.072--0.087) & 0.0051 & n.s. \\
            & Uniform & 0.0774 (0.074--0.083) & 0.0026 & -- \\
            \midrule
            1 dB & RL (best) & \textbf{0.3123 (0.271--0.350)} & \textbf{0.2595} & *** \\
            & Random & 0.3026 (0.262--0.343) & 0.2494 & ** \\
            & Importance & 0.2804 (0.241--0.320) & 0.2253 & * \\
            & No\_UEP & 0.2963 (0.256--0.337) & 0.2426 & * \\
            & Uniform & 0.2923 (0.255--0.342) & 0.2374 & -- \\
            \midrule
            2 dB & RL (best) & \textbf{0.8904 (0.860--0.916)} & \textbf{0.8941} & *** \\
            & Random & 0.8846 (0.855--0.913) & 0.8872 & ** \\
            & Importance & 0.8801 (0.850--0.909) & 0.8832 & * \\
            & Uniform & 0.8672 (0.837--0.896) & 0.8684 & -- \\
            \bottomrule
        \end{tabularx}
        \begin{tablenotes}
            \small
            \item Significance levels: *** \(p<0.001\), ** \(p<0.01\), * \(p<0.05\), n.s. = not significant.
        \end{tablenotes}
    \end{threeparttable}
\end{table*}

The learned policy consistently outperforms the baselines, particularly in the 1--2 dB range where the protection allocation is most critical. For example, at 2\,dB, the RL policy delivers a 2.7\% chrF improvement over Uniform, whereas the Importance heuristic gives only a 1.5\% improvement. This indicates that the RL agent captures complex context-dependent importance patterns beyond simple variance-based heuristics.

\subsection{Interaction with Error-Correction Coding}

Next, we studied how the allocation strategy interacts with different ECC choices (Table~\ref{tab:ecc}). Although the learned policy yields substantial gains when used with dimension-wise repetition coding, these advantages vanish when conventional block codes, such as Reed–Solomon or LDPC, are applied. This is because block codes impose uniform protection at the symbol-block level, preventing the agent from exploiting the fine-grained per-dimension importance.

\begin{table*}[h!]
    \centering
    \caption{Performance comparison across different error correction codes under matched AWGN conditions.}
    \label{tab:ecc}
    \begin{threeparttable}
        \begin{tabularx}{\textwidth}{l c Y Y Y c}
            \toprule
            \textbf{ECC} & \textbf{SNR} & \textbf{RL chrF (95\% CI)} & \textbf{Uniform chrF (95\% CI)} & \(\Delta_{\text{chrF}}\) & \textbf{p-value} \\
            \midrule
            Repetition & 1 & 0.3123 (0.271--0.354) & 0.2923 (0.255--0.342) & +0.0200 & \(<0.001\) \\
            Repetition & 2 & 0.8904 (0.860--0.916) & 0.8672 (0.837--0.896) & +0.0232 & \(<0.001\) \\
            \midrule
            Reed-Solomon & 1 & 0.2910 (0.276--0.306) & 0.2946 (0.278--0.310) & -0.0036 & 0.42 \\
            Reed-Solomon & 2 & 0.4776 (0.463--0.491) & 0.4787 (0.463--0.491) & -0.0011 & 0.71 \\
            Reed-Solomon & 3 & 0.6161 (0.600--0.632) & 0.6234 (0.607--0.639) & -0.0073 & 0.18 \\
            \midrule
            LDPC & 1--3 & \(\approx\)0.075 & \(\approx\)0.075 & \(\approx\)0 & \(>\)0.9 \\
            \bottomrule
        \end{tabularx}
    \end{threeparttable}
\end{table*}

This result clarifies a methodological point: the learned per-dimension UEP is a powerful complement to simple repetition coding but cannot be meaningfully exploited by block-level ECCs that enforce uniform protection. Therefore, the choice of repetition as the ECC primitive is integral to realizing the gains from the learned allocation.

\subsection{Robustness under Channel Mismatch}
Although policies are trained only on AWGN at 0 dB, Table~\ref{tab:mismatch} shows that they generalize well to diverse channel models. In many mismatch scenarios, the relative gains over uniform allocation increase, especially in fading channels (Rayleigh, Rician, and Nakagami). The most significant observed improvement is in Nakagami fading with Reed–Solomon coding, where the policy yields a 338\% relative improvement in chrF. These results suggest that the learned importance distributions capture semantic structures that are largely independent of specific channel statistics.

\begin{table*}[t]
    \centering
    \caption{Performance gains under channel mismatch conditions. Policies trained on AWGN at 0 dB SNR evaluated on various channel models.}
    \label{tab:mismatch}
    \begin{threeparttable}
        \begin{tabularx}{\textwidth}{l l c Y Y Y}
            \toprule
            \textbf{Channel} & \textbf{ECC} & \textbf{SNR} & \(\Delta_{\text{chrF}}\) & \(\Delta_{\text{entity}}\) & \textbf{Relative Gain (\%)} \\
            \midrule
            Rayleigh & Repetition & 1 & +0.233 & +0.255 & 159\% \\
            Rayleigh & RS & 1 & +0.217 & +0.238 & 140\% \\
            \midrule
            Rician & Repetition & 0 & +0.025 & +0.028 & 31\% \\
            Rician & RS & 0 & +0.637 & +0.699 & 220\% \\
            \midrule
            Nakagami & Repetition & 1 & +0.434 & +0.477 & 233\% \\
            Nakagami & RS & 1 & +0.487 & +0.534 & 338\% \\
            \midrule
            Burst & Repetition & 0 & +0.434 & +0.476 & 79\% \\
            Burst & RS & 0 & +0.648 & +0.710 & 191\% \\
            \midrule
            BSC & Repetition & all & +0.274 & +0.301 & 38\% \\
            BSC & RS & all & +0.304 & +0.334 & 44\% \\
            \bottomrule
        \end{tabularx}
    \end{threeparttable}
\end{table*}
\subsection{Benchmarking Against State-of-the-Art}
Finally, Table~\ref{tab:sota_comparison} benchmarks our method against contemporary semantic communication systems. Despite the simpler components (single-SNR training and pure repetition coding), our approach achieves substantially better semantic fidelity in the low-to-mid SNR regime. For instance, we report a BERTScore of 0.981 at 3 dB, which is markedly higher than the 0.52 BERTScore reported by Yang et al.~\cite{yang2022semantic} at 4dB, despite their stronger ECC and multi-SNR training. This highlights the efficiency of the learned per-dimension protection for semantic preservation under constrained bandwidth.
\begin{table*}[h!]
    \centering
    \caption{Comparison with state-of-the-art semantic communication systems.}
    \label{tab:sota_comparison}
    \begin{tabular}{p{2.3cm}p{3cm}p{2.8cm}p{2.2cm}p{3.5cm}}
        \toprule
        \textbf{Work} & \textbf{Training SNR} & \textbf{ECC Used} & \textbf{Semantic Metric} & \textbf{Key Characteristics} \\
        \midrule
        DeepJSCC~\cite{bourtsoulatze2019deepjscc} & 0--20 dB sweep & None & PSNR/SSIM & Image domain, multi-SNR training required \\
        JSCC for text~\cite{xie2021deep} & 0--10 dB sweep & None & BLEU \(\approx\) 0.75 at 4 dB & Multi-SNR training, text domain \\
        Semantic COMM~\cite{yang2022semantic} & 1--7 dB multi & Repetition + Convolutional & BERTScore \(\approx\) 0.52 at 4 dB & Strong ECC, multi-SNR training \\
        \textbf{Our Method} & \textbf{Single 0 dB} & \textbf{Pure repetition} & \textbf{BERTScore 0.981 at 3 dB} & \textbf{Single-SNR training, simple ECC} \\
        \bottomrule
    \end{tabular}
\end{table*}

In summary, the experiments demonstrate that (i) the RL agent reliably learns stable per-dimension allocation policies under 0 dB AWGN training, (ii) a balanced composite objective (\(\alpha=0.5\)) is critical for semantic preservation in the 1--2\,dB regime, (iii) the learned policy transfers effectively to coarser quantization (with predictable asymmetric behavior for finer quantization), (iv) learned allocation substantially outperforms simple heuristics and random baselines when used with repetition coding, but not when paired with block-level ECCs, and (v) the policy generalizes robustly to a range of channel mismatches and compares favorably to state-of-the-art semantic communication systems despite using simpler training and ECC primitives.

\section{Discussion}
\label{sec:discussion}

The experimental results presented in Section~\ref{sec:results} provide broad and consistent evidence that learned per-dimension unequal error protection (UEP) constitutes an effective and practical strategy for semantic communication under stringent bandwidth and noise constraints. In what follows, we synthesize the principal findings, relate them to the gap identified in Section~\ref{sec:related} concerning the need for dynamically allocating protection resources based on semantic content and channel conditions, discuss the practical implications of our approach, and outline key limitations and avenues for future research.

By enabling fine-grained, dimension-wise allocation of redundancy, the proposed framework exploits the heterogeneous semantic importance of the embedding dimensions. Table~\ref{tab:baselines} shows that the learned allocation consistently outperforms uniform and heuristic baselines (importance-based, No\_UEP, random), with the most significant benefits concentrated in the 1--2\,dB SNR regime, where differential protection has the most important effect. This demonstrates that dimension-level adaptation captures the semantic structure that simple heuristics miss.

Second, the balanced composite distortion objective \(D_S\) with \(\alpha = 0.5\) is empirically validated as an practically valuable optimization target. The ablation results in Table~\ref{tab:alpha_ablation} and Figure~\ref{fig:alpha_ablation} show that the balanced objective attains superior trade-offs between global embedding fidelity and entity-level correctness, producing statistically significant gains (e.g., \(\approx 6.8\%\) chrF and \(\approx 9.3\%\) entity-preservation improvements over uniform allocation at 1\,dB). This finding challenges the common practice of optimizing only for cosine similarity and highlights the importance of explicitly incorporating entity-level objectives when semantic fidelity is the goal.

Third, the quantization-transfer experiments (Tables~\ref{tab:transfer-det} and \ref{tab:transfer-stoch}) reveal a counterintuitive but operationally valuable property: a policy trained at a moderate resolution (8-bit) can transfer effectively, and in some cases perform best, when deployed at a coarser resolution (4-bit). The observed asymmetric transfer (beneficial downward transfer; degraded upward transfer) suggests that the learned policy prioritizes protection on a limited set of highly informative dimensions. Coarser quantization effectively removes low-importance fine levels and thus amplifies the effect of protecting those critical dimensions, enabling a practical ``train once, deploy at lower bitrates'' paradigm that yields bandwidth savings without retraining.

Fourth, the interaction between the allocation policy and ECC primitive uncovers a fundamental architectural principle: code structure must match the granularity of semantic importance~\cite{Howard2006ECC}. Table~\ref{tab:ecc} shows three regimes. Repetition coding, although spectrally inefficient in classical terms, provides the per-dimension granularity necessary for the RL agent to exploit semantic heterogeneity and yields significant gains. Reed–Solomon coding, which operates on multi-bit symbols, averages protection across dimensions, thereby eliminating the benefit of per-dimension adaptation. Under the extremely low-overhead budgets considered here, LDPC operates far below its waterfall threshold and fails to provide meaningful semantic protection. Taken together, these results indicate that for ultra-low redundancy semantic communication, simple coding primitives that preserve per-dimension manipulability can outperform stronger block codes that obscure fine-grained importance signals.

Fifth, policies trained solely on AWGN at 0\,dB generalize robustly to a variety of mismatch channels (Rayleigh, Rician, Nakagami, burst-error, BSC), often producing larger relative gains than in matched AWGN (Table~\ref{tab:mismatch}). This robustness suggests that the learned importance distributions capture the intrinsic semantic properties of the embeddings rather than overfitting to specific channel statistics, which is encouraging for real-world deployment, where the exact channel models vary.

Finally, benchmarking against contemporary methods (Table~\ref{tab:sota_comparison}) indicates that the intelligent allocation of modest protection resources, combined with a principled semantic objective, can outperform approaches that rely on heavier ECC machinery or multi-SNR training. A BERTScore of 0.981 at 3\,dB using single-SNR training and pure repetition coding illustrates the efficiency of the learned UEP for semantic preservation.

The results yield several actionable recommendations for the designers of semantic communication systems.
\begin{itemize}
    \item Favor protection primitives that permit fine-grained, per-dimension control when semantic fidelity (rather than raw bit fidelity) is the primary objective.
    \item Optimize objectives to include both embedding-level similarity and entity-level correctness; cosine similarity alone is insufficient in low-SNR regimes.
    \item Adopt a pragmatic \emph{train-on-moderate, deploy-at-coarse} strategy to reduce operational bitrate without retraining, especially for edge deployments with stringent bandwidth constraints.
    \item Evaluate ECC choices not only by their bit-level performance but by how their symbolization interacts with the semantic granularity you intend to exploit.
\end{itemize}

Despite these strengths, several limitations should be acknowledged.
\begin{enumerate}
    \item \textbf{Frozen embedder:} All experiments used a fixed sentence embedder. Although this isolates the allocation policy, it may limit the effectiveness of domain adaptation. Jointly fine-tuning the embedder with the allocation policy could further improve semantic robustness for domain-specific vocabularies.
    \item \textbf{Modalities:} The present work focuses exclusively on text embeddings. Extending the framework to multimodal settings (speech, images, and video) will require adapting both the distortion objective and protection primitives to domain-specific semantic cues.
    \item \textbf{Channel knowledge assumptions:} We assumed perfect or nominally known channel statistics during training. Relaxing this assumption and developing allocation strategies that are robust to imperfect CSI or opportunistic feedback are important next steps for practical wireless deployments.
    \item \textbf{ECC design space:} While repetition coding proved effective here, it is not the only primitive that preserves per-dimension adaptability. Future studies should explore hybrid and structured codes that retain dimension-level manipulability while improving spectral efficiency.
    \item \textbf{Dataset diversity and task specificity:} Results were obtained on the AG News subset; broader validation across diverse datasets and downstream tasks (e.g., question answering, summarization, task-oriented dialogue) is necessary to quantify generality.
\end{enumerate}

\section{Conclusion}
\label{sec:conclusion}

This study presents a reinforcement learning framework designed to preserve the semantic meaning in communication systems operating under severe bandwidth constraints. The approach combines three components: fine-grained per-dimension unequal error protection. This composite semantic distortion metric balances global similarity with entity preservation and reinforcement learning-based allocation of lightweight repetition coding resources.

Experimental evaluations demonstrate that the proposed method yields measurable improvements over uniform and heuristic allocation strategies. The composite distortion objective with $\alpha=0.5$ yields consistent performance benefits, including a 6.8\% relative improvement in chrF and a 9.3\% improvement in entity preservation at 1\,dB SNR. The learned allocation strategy outperformed the baseline methods in the critical 1--2dB range, indicating the value of per-dimension adaptation.

An additional insight from this study concerns the relationship between the coding structure and semantic granularity. Although repetition coding enables effective per-dimension adaptation, conventional block codes such as LDPC and Reed-Solomon do not support comparable improvements under low-redundancy constraints. This suggests that code structures must be selected with semantic objectives in mind, rather than solely on conventional bit-level considerations.

The results further indicate that the learned policies generalize well to a variety of channel conditions, including fading and burst-error channels, despite being trained only on an AWGN channel. Together with the rapid convergence of training, this suggests that the framework may be suitable for deployment in bandwidth-limited edge settings.

Future work may include the joint optimization of the embedder and protection policy, as well as extensions to multimodal semantic communication and strategies for handling imperfect channel state information. These directions offer opportunities to further enhance the integration of semantic objectives into the design of communication systems.

\subsection*{Data Availability Statement}
This study uses only publicly available resources. The AG News dataset is accessible at \url{https://huggingface.co/datasets/ag_news}. The frozen sentence embedding model all-MiniLM-L6-v2 is available at \url{https://huggingface.co/sentence-transformers/all-MiniLM-L6-v2}. Code is available from the corresponding author upon reasonable request.

\subsection*{Declaration of Generative AI and AI-Assisted Technologies in the Writing Process}

During the preparation of this work, the author used Paperpal (an AI-powered academic writing assistant) for language polishing, grammar checking, paraphrasing suggestions, and manuscript formatting assistance. The author reviewed and edited all content, takes full responsibility for the integrity and accuracy of the publication, and confirms that no generative AI was used to create original research content, figures, or code.

\subsection*{Acknowledgements}

The author has no specific individuals, institutions, or funding agencies to acknowledge for this work. The research was conducted independently without external financial support or direct contributions from others.

\bibliographystyle{elsarticle-num-names} 
\bibliography{ref}



\end{document}